\documentclass[letterpaper,10pt,conference]{ieeeconf}

\IEEEoverridecommandlockouts
\usepackage{multicol}
\usepackage[bookmarks=true]{hyperref}
\usepackage{algorithm}
\usepackage{algpseudocode}
\usepackage{amsmath}
\usepackage{amssymb}
\usepackage{graphicx}
\usepackage{mathtools}
\usepackage{caption}
\usepackage{subcaption}
\usepackage{algorithm}
\usepackage{algpseudocode}
\usepackage{tikz}
\usepackage{times}
\usepackage{color}
\usepackage{cite}

\usepackage[font=footnotesize]{caption}

\usetikzlibrary{calc,shapes}

\newtheorem{theorem}{Theorem}

\newtheorem{definition}{Definition}
\newtheorem{proposition}{Proposition}
\newtheorem{note}{Note}

\begin{document}
\title{\LARGE \bf
Using Information Invariants to Compare Swarm Algorithms and General Multi-Robot Algorithms
}
\author{Gabriel Arpino, Kyle Morris, Sasanka Nagavalli, \textit{Student Member, IEEE}, Katia Sycara, \textit{Fellow, IEEE}
\thanks{*The authors are with the Robotics Institute, School of Computer Science, Carnegie Mellon University, Pittsburgh, Pennsylvania, USA. Email: {\tt \small \{garpino, kmorris1, snagaval\}@andrew.cmu.edu, katia@cs.cmu.edu}. This research has been sponsored in part by AFOSR Award FA9550-15-1-0442 and an NSERC PGS D scholarship.}
}
\maketitle

\begin{abstract}
Robotic swarms are decentralized multi-robot systems whose members use local information from proximal neighbors to execute simple reactive control laws that result in emergent collective behaviors. In contrast, members of a general multi-robot system may have access to global information, all-to-all communication or sophisticated deliberative collaboration. Some algorithms in the literature are applicable to robotic swarms. Others require the extra complexity of general multi-robot systems. Given an application domain, a system designer or supervisory operator must choose an appropriate system or algorithm respectively that will enable them to achieve their goals while satisfying mission constraints (e.g. bandwidth, energy, time limits). In this paper, we compare representative swarm and general multi-robot algorithms in two application domains --- navigation and dynamic area coverage --- with respect to several metrics (e.g. completion time, distance travelled). Our objective is to characterize each class of algorithms to inform offline system design decisions by engineers or online algorithm selection decisions by supervisory operators. Our contributions are (a) an empirical performance comparison of representative swarm and general multi-robot algorithms in two application domains, (b) a comparative analysis of the algorithms based on the theory of information invariants, which provides a theoretical characterization supported by our empirical results.
\end{abstract}

\section{Introduction}

For more than the past decade, there has been significant growing interest in multi-robot systems (MRSs), which use a team of cooperating robots to accomplish tasks collaboratively. This enables multi-robot systems to complete tasks that cannot be completed by a single robot or would be less efficient or effective to complete with a single robot~\cite{yan2013survey}. Multi-robot systems may be characterized along multiple dimensions~\cite{iocchi2000reactivity} including but not limited to the mechanism for coordination among robots (e.g. communication vs. sensing only), centrality of coordination (e.g. centralized vs. decentralized), the extent of information available to team members (e.g. local information vs. global information), sophistication of the control logic executed by robots (e.g. reactive vs. deliberative collaboration), the structure of information propagation within the robotic network (e.g. neighbor-only connectivity vs. all-to-all communication), homogeneity or heterogeneity of the team members.

Algorithms for multi-robot coordination considered in the extant literature~\cite{bullo2009distributed, kolling2016human} make different combinations of assumptions along these dimensions about the underlying multi-robot system to which they are applied. Conversely, the extensive behavior of the multi-robot system is dictated by the properties of the algorithm it uses for internal coordination. Now we consider a particular type of multi-robot system known as a \textit{robotic swarm}. Robotic swarms are characterized by homogeneous robots executing a simple reactive control law using only local information from proximal swarm members and the environment within a limited spatial neighborhood. The collective behavior (e.g. flocking, rendezvous, dispersion) of the swarm emerges as a result of all swarm members executing the same local control law and no individual swarm member ever necessarily becomes aware of the whole swarm. Swarm behaviors are often not goal-directed but may be combined through behavior composition to accomplish tasks for which no individual behavior was designed~\cite{nagavalli2017automated, nagavalli2017on}. Conventional wisdom in the literature has been that the simplicity of the local control laws executed by swarm members and the locality of the information required for their execution makes robotic swarms more scalable and robust than other types of multi-robot systems because members may be inserted and deleted with minimal system reconfiguration~\cite{kolling2016human}.

In this paper, we compare the performance of algorithms designed specifically for robotic swarms (i.e. systems whose members have access to local information, neighbor-only sensing and simple reactive control laws) to those designed for general multi-robot systems (i.e. systems whose members may have access to global information, all-to-all communication and sophisticated deliberative collaboration). Since all decentralized algorithms can equivalently be implemented in a centralized system, to make the comparison useful, in both cases, we assume coordination is decentralized. Algorithms have been developed for several multi-robot applications~\cite{yan2013survey} including navigation~\cite{desaraju2011decentralized, alonso2016distributed}, static area coverage~\cite{breitenmoser2010voronoi}, dynamic area coverage~\cite{choset2001coverage, galceran2013survey, nestmeyer2017decentralized, atincc2014swarm}, patrolling~\cite{portugal2011survey} and many more. However, in this paper, we limit our comparison to representative algorithms in two application domains: (a) navigation and (b) dynamic area coverage.

Our objective in performing this comparison is to identify and highlight the strengths and weaknesses of swarm and general multi-robot algorithms in a manner that can inform offline design decisions made by engineers or online operational decisions made by supervisory operators of multi-robot systems. For example, for a given application domain, an engineer may decide a simple robotic swarm will suffice without the additional cost and complexity of a general multi-robot system. Conversely, given a multi-robot system, there may be situations where it is beneficial (e.g. reduced coordination complexity) for a supervisory operator to apply an algorithm designed for robotic swarms rather than incur the communication overhead of a general multi-robot coordination algorithm. Other examples exist and it is our hope that the results in this paper can inform such decisions.

We make the following contributions: (a) an empirical evaluation of the performance of representative swarm and general multi-robot algorithms within two different application domains (navigation and dynamic coverage), (b) a comparative analysis of these algorithms based on the theory of information invariants, which provides a  theoretical characterization supported by our empirical results.


\section{Swarm and Multi-Robot Algorithm Selection}
\label{sec:algorithm_selection}

\subsection{Algorithm Description and Implementation}

As shown in Table~\ref{tbl:algorithms}, five representative algorithms were implemented to study the relative performance of swarm and decentralized multi-robot algorithms. The algorithms were then analyzed in terms of their information invariants \cite{donald1995information} and their empirical performance was compared with respect to multiple metrics.

\begin{table*}[!ht]
\caption{Algorithms for Comparison}
\label{tbl:algorithms}
\begin{center}
\begin{tabular}{| c | c | c |}
\hline
 & \bf{Swarm} & \bf{Multi-Robot}\\
\hline
\bf{Navigation} & Potential Fields~\cite{li2013bounded} & DMA-RRT~\cite{desaraju2011decentralized}\\
 & Proportional Barrier Certificates~\cite{borrmann2015control, wang2016safety} & \\
 \hline
\bf{Dynamic Area Coverage} & Individual Dynamic Coverage~\cite{hussein2007effective} & Group Dynamic Coverage~\cite{atincc2014swarm}\\
\hline
\end{tabular}
\end{center}
\end{table*}



\subsubsection{Potential Fields (PF)}
This algorithm, as described in \cite{li2013bounded}, is a gradient-based navigation approach with guaranteed goal convergence. In implementation, the algorithm exhibited numerical issues involving the magnitude of the gradient, discussed in the results section. Characteristics of this algorithm relevant to our comparison: a) no explicit communication between robots (sensing-only swarm algorithm), b) robots initially configured as a connected graph will converge to the goal as a connected graph, c) robots follow the gradient of a potential function with a global minima at the goal while avoiding obstacles and other robots.

\subsubsection{Proportional Barrier Certificates (PBC)}

PBC uses a proportional controller in conjunction with a barrier certificates reactive controller as described in \cite{borrmann2015control, wang2016safety}. The algorithm uses a quadratic programming approach enabling obstacle avoidance while following the original goal trajectory (proportional controller output). The barrier certificates algorithm guarantees forward invariance of the safe set, implying that inter-robot collision avoidance is guaranteed. Characteristics of this algorithm relevant to our comparison: a) no explicit communication between robots (sensing-only swarm algorithm), b) robots take action to avoid collisions only when sufficiently close (sense the proximity), c) robots follow a proportional controller, creating a straight line path to the goal when there are no obstacles.






\subsubsection{DMA-RRT}
The DMA-RRT algorithm outlined in~\cite{desaraju2011decentralized} embeds a closed-loop RRT \cite{luders2010bounds} in each robot and introduces a merit-based token passing coordination mechanism. Agents with the largest incentive to replan will acquire the token and broadcast their updated plan to all other agents. Other agents then forward simulate this time parameterized trajectory and update their own constraints. Our implementation uses a holonomic model with a proportional controller (PC) for waypoint navigation. It is important to note that we are using the original DMA-RRT algorithm and not the cooperative extension involving emergency stops \cite{desaraju2011decentralized}. Characteristics of this algorithm relevant to our comparison: a) explicit communication between the robots which is bounded above by a set of $n$ waypoints, where $n$ is the size of the created path, b) communication occurs through a token passing strategy, where only one robot gets to alter its plan at any given time, c) obstacle and inter-robot avoidance is done implicitly through the robots' creation of non-intersecting paths.

\subsubsection{Individual Dynamic Coverage (IDC)}
The dynamic area coverage problem tackled in this work is not to be confused with the static area coverage~(e.g. \cite{breitenmoser2010voronoi}) problem, where robots move to maximize coverage of map in the final positions to which they converge. Instead, it involves generating robot motion to maximize the coverage (sensed areas) of an entire map by multiple robots over the entire mission time horizon. Dynamic area coverage problems have been tackled by various stigmergic and frontier-based dynamic coverage algorithms such as~\cite{wang2011frontier,tang2017stigmergy}. The IDC algorithm is a gradient-based dynamic coverage method outlined in \cite{hussein2007effective}. The coverage error is guaranteed to be non-increasing but susceptible to local minima. Perturbations from local minima, in the computational space, involve the selection of a new point (according to a predefined rule) upon which to resume gradient descent. In the real world, this corresponds to an individual robot discontinuing the application of the gradient-based control law, physically moving to a new location and then resuming the application of the gradient-based control law. All robots are perturbed individually and the rule for new position selection was not specified. The algorithm was implemented in a decentralized way by replacing the centralized goal point ordering with an emergency timer as used in the provided implementation of \cite{hussein2007effective}. Characteristics of this algorithm relevant to our comparison are: a) no explicit inter-robot communication (sensing-only swarm algorithm), b) robots follow a gradient descent procedure towards unexplored areas, c) agents get stuck in local minima, d) agents exit local minima independent of other agents, e) each individual grid cell in the map has a limited amount of information that can be exhausted after repeated visits by robots.


\subsubsection{Group Dynamic Coverage (GDC)}
This decentralized algorithm involves the multi-robot gradient based dynamic coverage strategy outlined in~\cite{atincc2014swarm}. The algorithm is similar to Perturbation Dynamic Coverage, but in the real world agents are now perturbed all at the same time once a global condition is met (all robots begin to slow down sufficiently) and a leader agent selects a new rendezvous point towards which the entire group moves. Once the multi-robot group is within a certain radius of the newly selected point, all robots shift into dynamic coverage mode and the state machine is reset. In the computational space, this corresponds to all agents reaching an overall local minima, and them all being perturbed by relocating themselves to the same region of the workspace, resuming the gradient descent from there. The main difference between the IDC algorithm and the GDC algorithm is that the latter perturbs all robots at the same time and to the same new point, once the change in coverage of the map descends below a threshold. This method involves the communication of a new swarming point from the leader to every other agent in the group, which in the real world could be accomplished by wireless communication (for example). Characteristics of this algorithm relevant to our comparison are: a) explicit inter-robot communication, b) robots follow a gradient descent vector towards unexplored areas, c) agents are susceptible to local minima, d) agents exit local minima \textbf{dependent} on the new perturbation point that the leader agent selects, e) each individual grid cell in the map has a limited amount of information that can be exhausted after repeated visits by robots.

\section{Experimental Evaluation}
\label{sec:experimental_evaluation}

\subsection{Reproduction of Results and Scaling}
Algorithms were implemented in Python and later integrated into the CMUSWARM Framework \cite{RISS}. We first reproduced the results in the corresponding papers. For each algorithm, the number of robots were scaled for preliminary qualitative testing beyond the experiments in the original papers. As seen in Figure~\ref{scaling_plot2}, the time to convergence for Potential Fields, PBC, and GDC generally increase with the number of robots, whereas Individual Dynamic Coverage decreases.

\begin{figure}[!h]
  \includegraphics[scale=0.45]{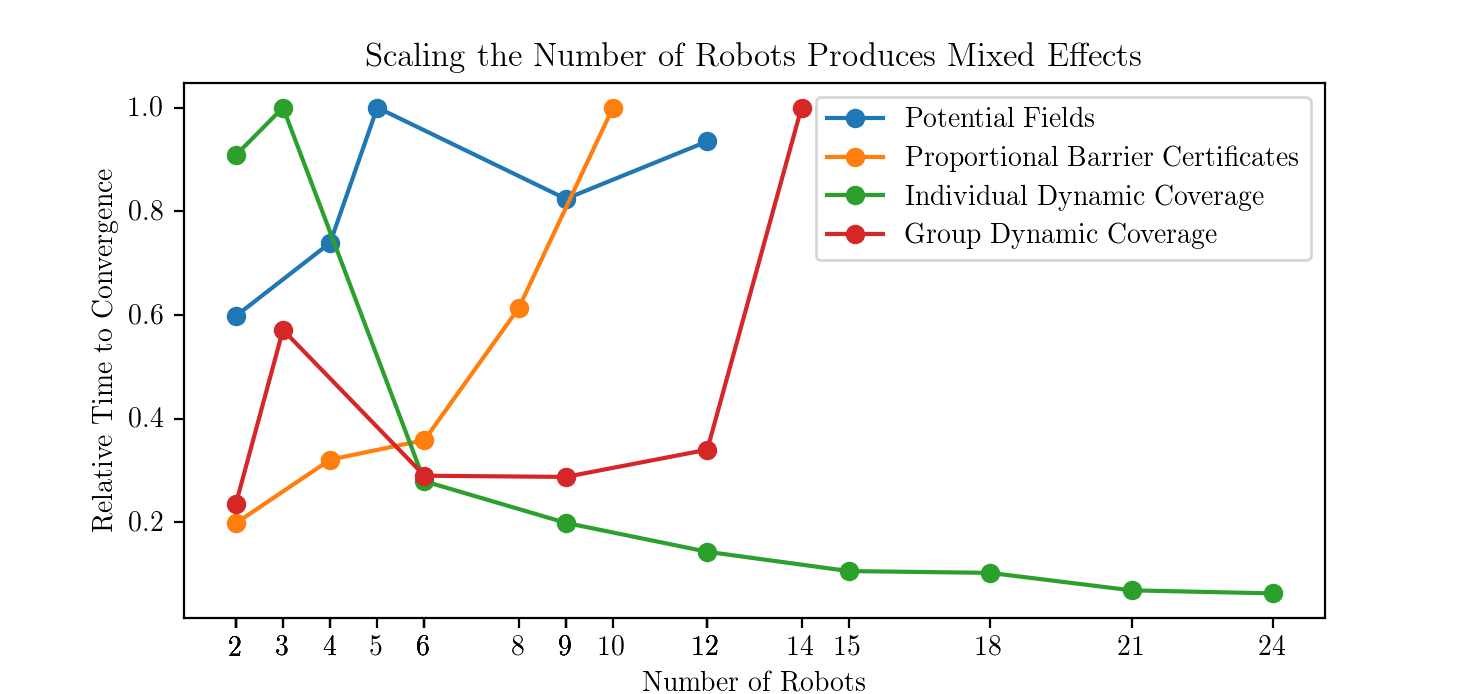}
  \centering
  \caption{Results of Preliminary Scaling}
  \label{scaling_plot2}
\end{figure}

Figure \ref{scaling_plot2} does not contain data on DMA-RRT due to the nature of the initial research. This algorithm was initially evaluated by having multiple robots cycle through a list of goals for 10 minutes. For DMA-RRT, we reproduced scenario A \cite{desaraju2011decentralized} of the paper which involved 10 agents cycling through 10 goals, and scenario B \cite{desaraju2011decentralized}, with 4 agents cycling through 2 goals. We observed in our reproduction that the higher the number of robots in the simulation, the more collisions were experienced. \footnote{Code used for simulations is available at https://github.com/gabrielarpino/Swarm-vs-Multi-Robot-Algorithms}


\subsection{Experimental Setup}
The algorithms mentioned above were benchmarked in the CMUSWARM Framework on ROS and Gazebo. Each algorithm was run for 20 trials on 5 different 20x20 meter maps, using 4, 8, and 16 holonomic 0.2 x 0.2 m iRobot Create vehicles totalling 300 trials.

An experimental trial is marked as successful when the algorithm converges and finishes the task within 10 minutes, otherwise it fails. We set a 10 minute cut-off primarily for three reasons. Firstly, we wish to establish a mission-critical time frame, in which the task must be completed. Secondly, the successful trials took far less than 10 minutes, with more than 10 minutes passing when a robot was stuck on an obstacle or in some other observably unrecoverable state. Lastly, with a 10 minute cap on simulation time, with a real time factor making the actual trial take 14+ minutes, running 1000+ trials to cover all 5 algorithms took well over 150 hours. We wish to optimize our Framework’s benchmarking process in the future to lessen this burden, and also introduce automated deadlock detection to determine if the system is in a stuck state.

If an algorithm fails to converge 20/20 times, we state it is intractable in the scenario, otherwise we will keep running trials until 20 successful trials are recorded.
Parameter tuning was a manual process performed for each algorithm in the initial 4 robot case. Once a feasible set of parameters was found, they were used in the 4, 8, and 16 robot scenario. Manual re-tuning was done for each map.

Performance is measured by the metrics of convergence time (until task completion), distance travelled, area coverage (for navigation), sensor coverage (for dynamic coverage), and number of collisions. These metrics are significant in mission-critical situations because optimal-time response is necessary in emergency settings, real robots use up fuel with movement, and because collisions can cause the robot system to break. Our measures of convergence time and distance travelled correspond to the dynamic coverage Time and dynamic coverage Cost metrics outlined in \cite{yan2015metrics}, and are therefore classified as objective methods of comparison for dynamic coverage algorithms. For navigation, convergence time is the amount of time required for all robots to become within 3 meters of the goal location. For dynamic coverage algorithms, convergence time is measured as the amount of time to reach 100\% coverage of the map.


\begin{figure}[h]
  \includegraphics[scale=0.19]{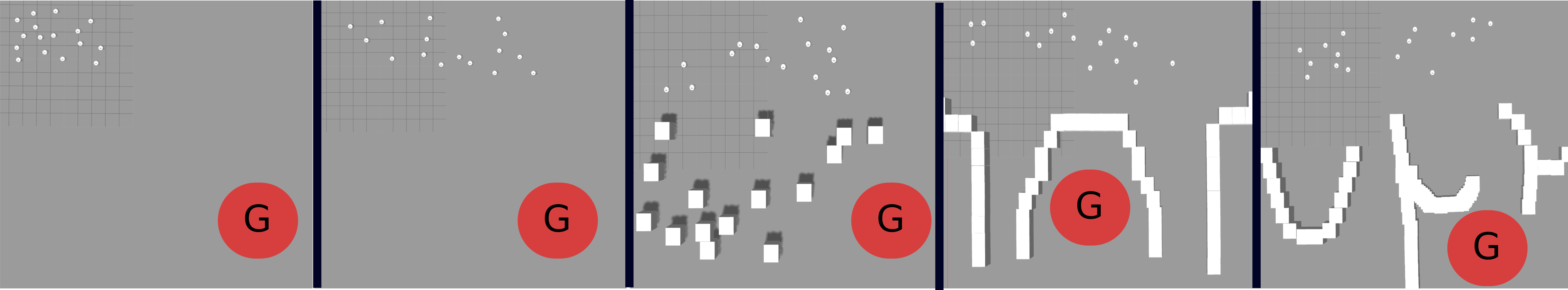}
  \centering
  \caption{The 5 maps in Gazebo used for our experiments. Maps From left to right\: 1 (Empty Map Dense with spawn region ([1-6], [1-8]), 2 (Empty Map Spread with spawn region ([1-6], [1-19])), 3 (Uniform Map), 4 (Corridor Map), 5 (Concave Map). The goal region is the red disk labeled G, with robots (White) spawning in the upper region. Robots spawn in the x in [1,6], y in [1,19] range for all maps except empty dense. }
  \label{experimental_maps}
\end{figure}

\subsection{Experimental Results}

\textbf{Potential Fields Navigation:} Much parameter tweaking was needed in order to have both convergence time and collision avoidance guarantees. PF had 0 inter-robot collisions in obstacle free environments, and the fewest collisions in obstacle filled environments (map 3, Figure \ref{experimental_maps}) among the navigation algorithms, however this algorithm becomes intractable as the number of robots surpasses roughly 8 in any scenario. We believe this is due to the connectivity constraints requiring all robots to be connected within some proximity radius. As more robots are introduced in a spread out scenario (maps 2-5) there is an increasing likelihood of breaking the initial connectivity based on initial spawn location. Furthermore as the number of robots or other static obstacles increases, the system becomes increasingly slow at maneuvering through obstacles (Figure \ref{total_convergence_time}). Thus, PF was intractable on the corridor and concave maps with our minimum scenario of 4 robots, as the purely reflexive behaviour results in deadlocks with obstacles. We refrained from scaling further on these maps as a result.

\textbf{Proportional Barrier Certificates Navigation:} The PBC controller was tractable on only the empty maps (1,2) but resulted in the least distance travelled, had the lowest convergence time, and covered the least area in it’s path among all navigation algorithms. The algorithm had few collisions in empty maps (1,2), slightly more than PF; however as the number of robots scaled to 16, it had the most collisions of all algorithms (Figures \ref{total_collisions}), likely due to the violation of initial safe set conditions. This algorithm requires robots to spawn at least some threshold (denoted by parameter Ds) away from each other at spawn time. In map 1 (Figure \ref{experimental_maps}) where the robots are spawned in a dense region, this condition is violated in the 16 robot case. Furthermore, there is a chance of this condition being violated in the second map despite a wider initial spawn region. A collision is recorded when the collision disk of 2 robots intersect for every second. Thus, if the initial safe-set guarantee is violated then two or more robots may collide and incur many collisions, skewing the results.

This algorithm required less parameter tweaking than PF; but was very specific to scenario. Further parameter tweaking as the number of robots increased may have prevented collisions; however we set our experimental evaluation to use the same set of parameters for all 3 cases of robots (4,8,16), and only changed parameters between different maps.

\textbf{DMA-RRT Navigation}: DMA-RRT was the most robust of all the algorithms, requiring minimal tweaking of parameters and being the only navigation algorithm that was tractable in every scenario.
This algorithm scaled well as the number of robots increased, almost not increasing in convergence time on empty maps and the uniform map (maps 1,2,3) (Figure~\ref{experimental_maps}) as the robots had alternative routes to take. On more cluttered environments (maps 4, 5) (Figure~\ref{experimental_maps}) the convergence time increased more significantly as the robots must take turns navigating through dense regions (Figure \ref{total_convergence_time}). An interesting observation is that in every scenario, the amount of distance travelled increased by less than a factor of 2 despite the number of robots doubling. One possible explanation is that with fewer robots the merit-based token is passed to each robot more frequently, meaning each robot is able to activate its current plan. With all robots acting at once some of them must take longer paths to avoid their neighbour's broadcasted trajectory which is an obstacle in the environment. However, when there are more robots the convergence time increases because the robots must take turns; however now the robots may take better paths as they aren't all navigating to the goal at once, many are stationary. Despite the theoretical guarantee of collision avoidance, collisions were present in our evaluation due to latencies in the ROS system and token exchange. Using the extended algorithm, Cooperative DMA-RRT, which involves emergency stops, may reduce the number of collisions when multiple robots face the same goal as it allows an agent to request others to stop if in close proximity.

\textbf{Individual Dynamic Coverage (IDC)}: The IDC algorithm had several parameters and required tuning. There were no collisions on empty maps (1,2) (Figure~\ref{experimental_maps}) without static obstacles, and the average number of collisions across all maps either remained the same or decreased as number of robots increased. One possible explanation is that with many robots in an environment, the coverage levels of the spawning region quickly increase which hinders many robots from exploring further. This results in less time for convergence, and less total distance travelled as only a few select robots scout the remaining portion of the environment (Figures \ref{exploration_convergence}, \ref{exploration_work}). Additionally, as the robots attract and repel each other by nature of the swarm algorithm, having more robots on the map may aid in preventing clustering of robots in obstacle dense regions.

\textbf{Group Dynamic Coverage}: GDC required more tuning than the swarm coverage algorithm. The behaviour was similar to that of the swarm dynamic coverage; however due to a leader robot guiding other robots to a specific region, significantly more distance was travelled as the whole group moved together and thus repeated coverage of each others area (Figure \ref{exploration_work}). This further resulted in a longer time for convergence.

\begin{figure}[h]
  \includegraphics[scale=0.092]{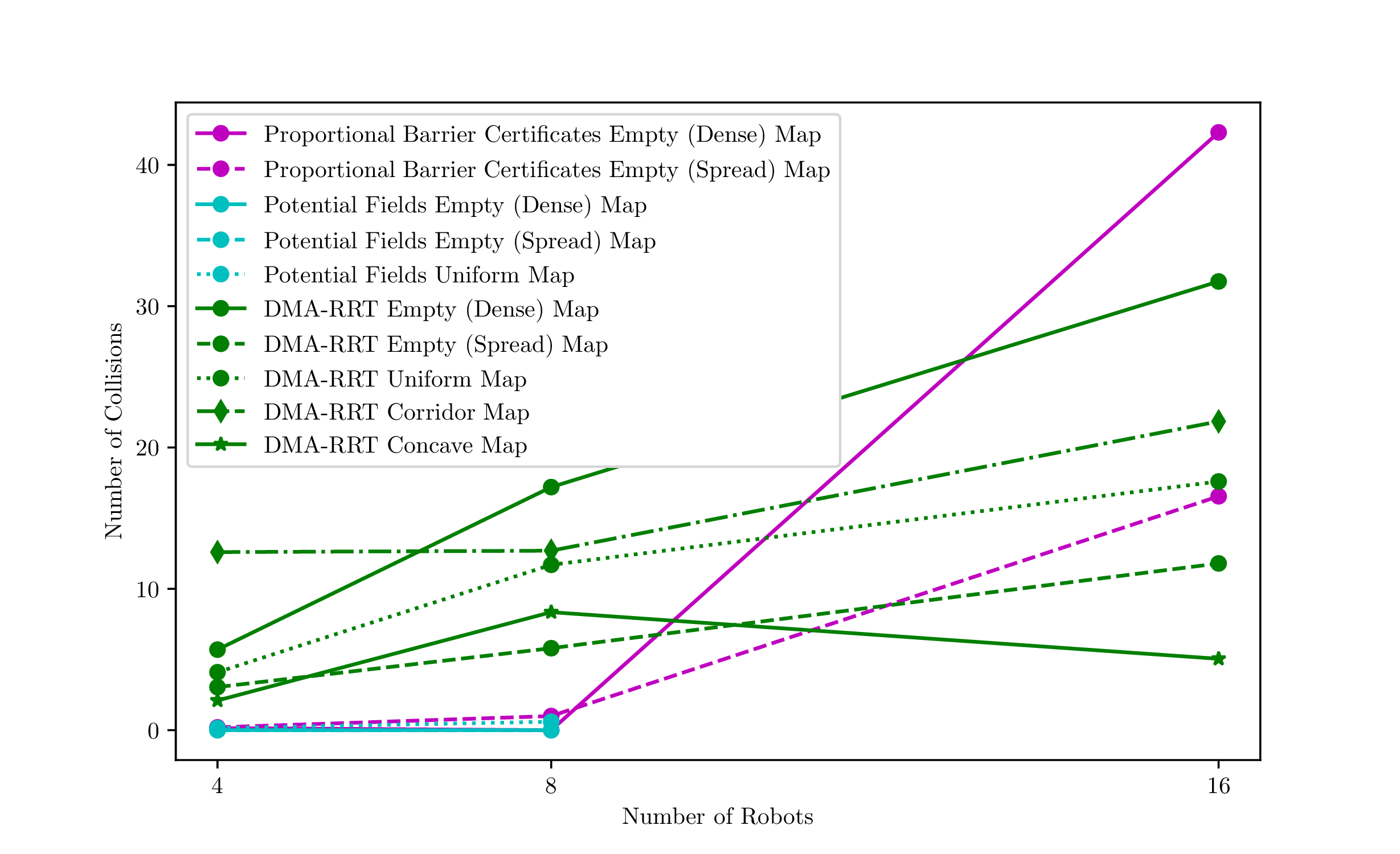}
  \centering
  \caption{Number of Collisions for Navigation Algorithms. Algorithms deemed infeasible for certain maps and numbers of robots were not plotted.}
  \label{total_collisions}
\end{figure}

\begin{figure}[h]
  \includegraphics[scale=0.092]{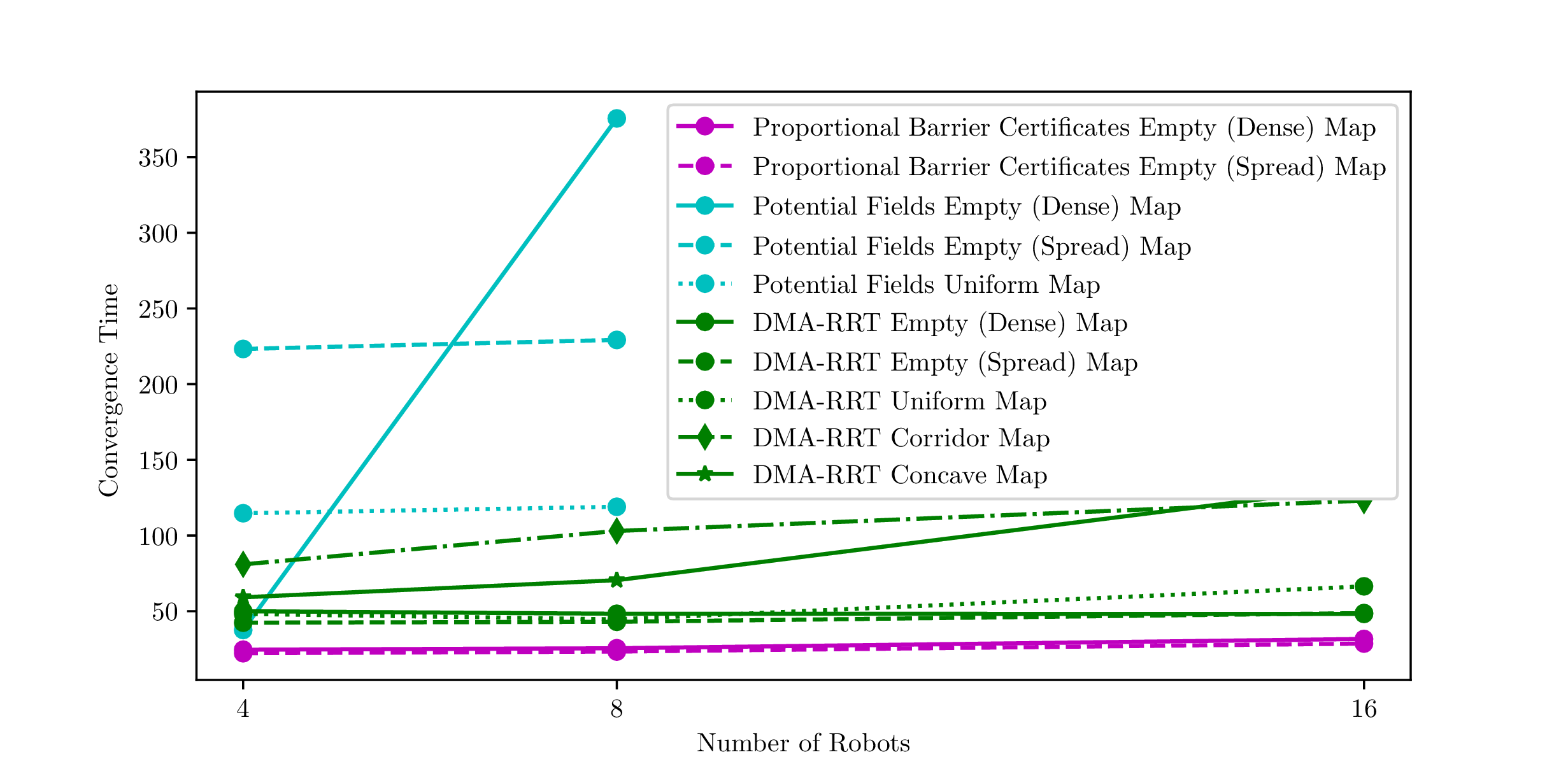}
  \centering
  \caption{Convergence Time for Navigation Algorithms}
  \label{total_convergence_time}
\end{figure}

\begin{figure}[h]
  \includegraphics[scale=0.10]{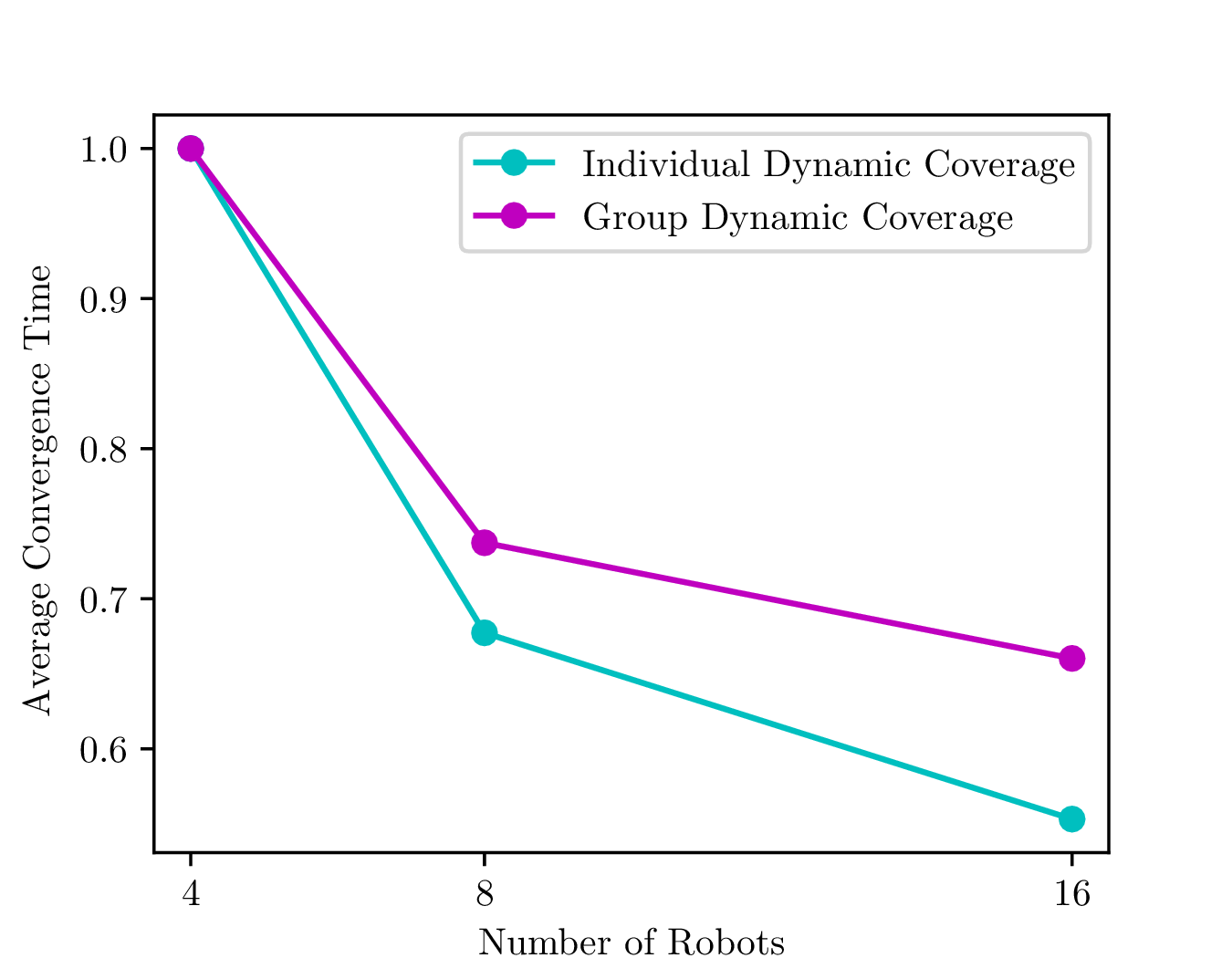}
  \centering
  \caption{Average Time to Convergence for the Dynamic Coverage Algorithms}
  \label{exploration_convergence}
\end{figure}


\begin{figure}[h]
  \includegraphics[scale=0.092]{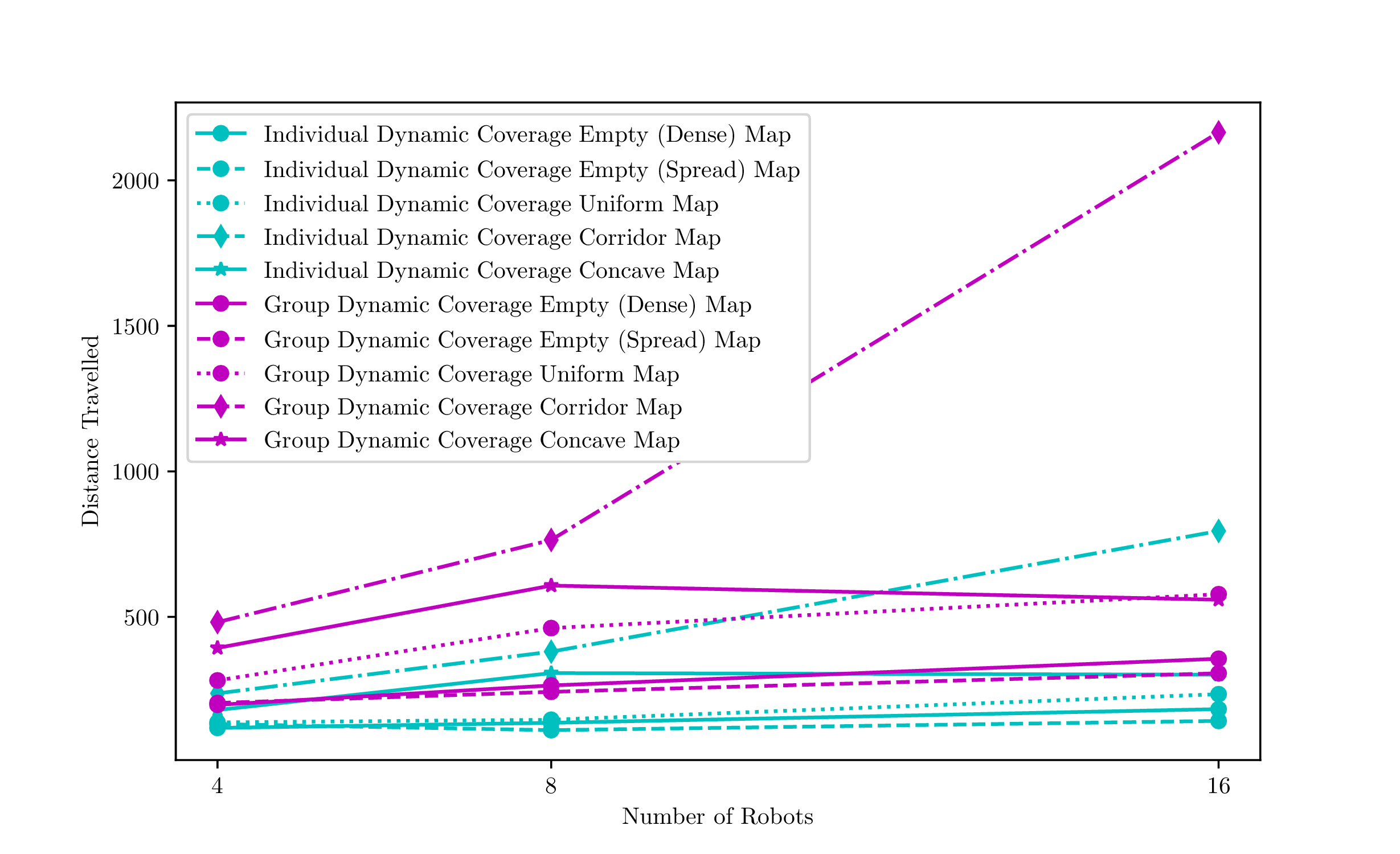}
  \centering
  \caption{Distance Travelled for both of the Dynamic Area Coverage algorithms on all map types.}
  \label{exploration_work}
\end{figure}





In contrast with the navigation task, IDC overall performed better than multi-robot dynamic coverage. IDC required less distance travelled, and time to converge in every scenario. Furthermore, IDC had the lowest number of collisions among any algorithm on the empty maps (1,2) (Figure \ref{experimental_maps}); while multi-robot dynamic coverage yielded less collisions on obstacle maps (3,4,5) (Figure \ref{experimental_maps}).

\subsection{Discussion}
In general, swarm navigation had less collisions than multi-robot navigation on every map it converged on with less than 16 robots (1,2,3) (Figure \ref{experimental_maps}); however the swarm algorithms quickly became intractable with obstacles and an increased number of robots. Furthermore, the swarm algorithms required significant tuning and it is cumbersome to explore the large parameter space. Finding a configuration that resulted in no collisions renders the convergence time extremely slow; but tweaking so there is convergence within reasonable time introduced collisions.

This means if all robots are closely grouped together, thus sensing each other, it would be equivalent in bandwidth to multi-robot dynamic coverage's $\sim 8.90$ kb/sec because we have 8 robots communicating to each other.
However, with swarm algorithms, the lower bound of bandwidth is actually still 0 kb/sec, the reason being if the robots are spread out far enough they won't sense each other, and thus won't receive the broadcast.

This highlights a fundamental contrast between the algorithms in our comparison: the swarm algorithms lack of explicit communication as opposed to multi-robot algorithms. Even with the case where sensing is treated as a form of communication, the swarm algorithms still do not have guaranteed communication due to limited sensing radius.

A final note of mention: despite each algorithm having theoretical guarantees for collision avoidance given some set of parameters, no such guarantee was satisfied during experimentation. We believe this is in part due to latency introduced by ROS, along with suboptimal parameter selection for the given scenario.

\section{Analysis}
\label{sec:analysis}

It was observed that, in the context of dynamic coverage, a swarm algorithm outperforms a multirobot algorithm in the metrics of distance travelled and time to convergence. In the context of navigation, however, this is not the case. The analysis of these findings is based on the Robotics Information Invariants work of \cite{donald1995information} and \cite{donald1997information}.


\subsection{Information Invariants Concepts}
The theory of Information Invariants for robotics was proposed in 1995\cite{donald1995information} and expands on the work of \cite{blum1978power}, where the task of maze searching by an automaton was proved to be solved in three different configurations involving a counter, pebble, or other agent. Since then, the theory has been used to analyze problems in robotic manipulation and control (ex: \cite{bohringer1997distributed}).

The theory serves to quantify the tradeoff between sensing, actuation, calibration, and communication in multi-robot systems. As in the work of \cite{donald1995information}, our groups of robots are modelled as circuit graphs. The following definitions will prove useful to our analysis:

\textbf{Definition 3.1} A set of sensors attached to computational and motor components able to alter their position in state space is termed a \textit{Sensor System}. A system of robots is an example of a \textit{Sensor System}.

\textbf{Definition 3.2} \textit{Circuit Graphs} are graphs representing the sensor system, $G = (V, E)$, where in the case of a system of robots the vertices $V$ are the sensori-motor components of the system (each individual robot) and the edges $E$ are the sensory or explicit communication connections between robots.

\textbf{Definition 4.1} \cite{donald1995information} For two sensor systems $S$ and $Q$ we say $Q$ simulates $S$ if the \textit{output} of $Q$ is the same as the \textit{output} of $S$. When viewing the sensors systems as dynamical systems, the \textit{output} of a sensor system refers to the limit set of the sensor system, a term in Dynamical Systems Theory referring to the state a dynamical system reaches after an infinite amount of time has passed. We write $S \cong Q$.

Our goal is to establish relations such as $S \cong Q + A$, where $S$ and $Q$ are swarm/multi-robot systems, and $A$ is an additional capability such as sensing or useful computation, implying that $S$ and $Q + A$ \textit{deliver equivalent information}.

A calibration of the system constrains the spatial relation between its various nodes, and a calibration required to install a sensor system (how it should be configured at the start of a mission to work) is termed an \textit{installation calibration}.

\textbf{Definition 5.1} \cite{donald1995information} Consider two sensor systems $S$ and $Q$. When $S$ and $Q$ require equivalent installation calibrations, and when the calibrations required to install $Q$ are necessary to specify $S$, we say $S$ dominates $Q$ in calibration complexity.

According to \textbf{Definition 6.2} in \cite{donald1995information}, for two sensor systems $J$ and $Q$ we write $J \leqslant Q$ when:
\begin{enumerate}
    \item $J$ simulates $Q$ ($Q \cong J$)
    \item $J$ dominates $Q$ in calibration complexity
    \item mb($Q$) is bounded above by mb($J$)
\end{enumerate}
where mb(.) denotes the maximum bandwidth of the system as in Definition 6.1 of \cite{donald1995information}:

\textbf{Definition 6.1.} \cite{donald1995information} We define the internal (resp. external) bandwidth of a sensor system $S$ to be the greatest bandwidth of any internal (resp. external) edge in $S$ (an edge represents explicit communication between two robots). We define the maximum bandwidth mb($S$) to be the greater of the internal bandwidth, external bandwidth, and the output size of $S$.

The intuition behind \textbf{Definition 6.2} is that $J$ as a sensor system would be reducible to $Q$: It can simulate $Q$ and is at least as complex in terms of calibration and overall bandwidth. A system of robots $J$ would be $k$-wire reducible to $Q$ if they satisfy \textbf{Definition 6.2} when $Q$ possesses $k$ additional pathways of communication between robots.

\subsection{Dynamic Area Coverage}

In this section, we try to establish an information reduction between IDC and GDC, so as to better explain our findings that multi-robot dynamic coverage does not always perform better than swarm dynamic coverage. Recall that IDC involves no explicit communication of a perturbation point, whereas GDC does.

Due to the gradient based nature of the two algorithms, the only information available to each robot (without explicit communication) at any given time are the own robot’s positions, and the discretized map with information of the coverage levels at every discretized point.

The discretized map used for multi-robot dynamic coverage tasks as in \cite{hussein2007effective, atincc2014swarm} has a fundamental limitation  in that the coverage value of a certain cell is capped at a maximum. This characteristic, as we will see, causes loss of information in the system and opens up the possibility for a hierarchy of information theoretic reductions in the task of multi-robot dynamic coverage. It is also a characteristic necessary for the task of dynamic coverage to occur, as the authors introduced it in order for the gradients to not flow towards already explored areas. If cells were not capped, a robot would never finish exploring a given cell. In the computational space, this refers to the fact that no additional information can be stored in the map's cell, and an agent acting upon this cell will cause no change. In the real world, this could correspond to, for example, robots marking the covered ground with paint, where the painted ground at some point saturates with so much paint that no change is visible.

Let us denote the multi-robot system used to solve the multi-robot dynamic coverage task as $M = (V, E)$, and the multi-robot system used to solve the IDC task as $S = (V', E')$. Following from definition 4.1 in \cite{donald1995information} $S$ and $M$ simulate each other because their limit sets are the same, therefore:

\begin{note}
$S \cong M$.
\end{note}


The multi-robot dynamic coverage system $M$ possesses the initial calibration requirement of guaranteeing that all robots are within a proximity radius for communication. This initial calibration requires the installation of n $2 DOF (x, y)$ robots such that they are all in communication proximity. The IDC system $S$, however, requires no initial calibrations, and therefore no installation requirements. Thus, $S$ and $M$ do not require equivalent installation calibrations.

Both dynamic coverage algorithms being analyzed involved a rule for perturbation point selection. These point selection rules are assumed to be \textit{deterministic}, implying that any randomness inherent in them is \textit{pseudorandom} and could be replicated through the knowledge of a \textit{random seed}. The possible rules for point selection can then be split into two families:

\begin{definition}
A Family 1 rule is a rule that only requires knowledge of the coverage levels of the map. A Family 2 rule is a rule that requires knowledge of the position of at least one other agent in the map.
\end{definition}

Following these definitions, we propose certain facts about our system that will be useful in proving our reduction:

\begin{proposition}
The state \textbf{q} of neighbouring robots is not always deducible from the coverage levels of the map.
\end{proposition}

\begin{proof}
This result stems from the fact that information is stored in the specified discretized map with the coverage values at every point in the map having a maximum value $Cstar$ (a parameter in the algorithm). We can then propose the situation where a neighbouring robot is situated in the middle of a fully covered area that is larger than the neighbouring robot’s coverage radius, so any change in the robot’s position will not increase the coverage levels of the map. Therefore, information is lost and the state of this robot cannot be deduced.
\end{proof}

\begin{proposition}
\begin{align*}
H_g + k \cdot comm(new\_point) \cong H_g + \sum_{i \in N} k \cdot comm(\bf{q}_i)
\end{align*}
where $H_g$ is a swarm gradient dynamic coverage system, $N$ is the set of all robots whose state is necessary by the rule for the new point calculation, and $k$ is the number of total robots in the system $-1$, because the leader communicates to all other robots in the system.
\end{proposition}

This proposition suggests that a swarm gradient dynamic coverage system with the leader communication of a new target point simulates a system with the communication of the states of all robots involved in the rule for new point selection. The minimum information communication required by the leader to achieve the same result are the states of all robots involved in the calculation of the new swarming location. With this proposition we are attempting to quantify exactly how much information is inherent in the GDC algorithm, and show that this is more than that present in IDC.

\begin{proof}
Case 1:
If the rule is of Family 1, then there are no robot positions involved and no communication is required at all, both sides are equal.
Case 2:
If the rule is of Family 2, then using Proposition 1 we know that something must be communicated in order for the new target point to be calculated since it involves the state of at least one of the other agents. The new target point for dynamic coverage can be calculated based on the knowledge of the point selection rule and the robot states involved in the calculation. Therefore, since these two information permutations are able to communicate the same new swarming target points, they are able to accomplish the same goals and consequently have the same limit sets, implying that they simulate each other.
\end{proof}


\begin{theorem}
GDC where the coverage levels of the map are bounded above is $|N| \cdot k$-wire reducible to IDC using the same map model. In other words,
\begin{align*}
M \leqslant S + \sum_{i \in N} k \cdot comm(\bf{q}_i)
\end{align*}
where $M$ is the multi-robot dynamic coverage system, $S$ is the IDC system, $k$ is the total number of agents -1, and $N$ is the number of agents involved in the new point selection rule.
\end{theorem}

\begin{proof}
As per Definition 6.2, let us denote the two systems $M$ and $S + \sum_{i \in N} k \cdot comm(\textbf{q}_i)$ as $J$ and $Q$ in the definition above.

1. \textit{$J$ simulates $Q$ ($Q \cong J$)}:
    It follows from Note 1 that $M$ simulates $S$ due to their matching limit sets. $J$ and $Q$ also have matching limit sets, because $S + \sum_{i \in N} k \cdot comm(\textbf{q}_i)$ is the system $S$ with an additional $|N| \cdot k$-wire communication of the robot states, which from Proposition 2 simulates $S + k \cdot comm(new\_point)$.

2. \textit{$J$ dominates $Q$ in calibration complexity}:
    It was mentioned that $M$ requires the calibrations of placing all agents within a proximity radius. $S$, however, requires no initial physical calibration relatively. When adding $\sum_{i \in N} k \cdot comm(\textbf{q}_i)$ as an information component to $S$, a calibration is now required in order to ensure proper communication of the swarm due to proximity constraints. An additional \textit{information processing calibration} (nonphysical) is required, and that is of programming the desired \textit{swarming rule} into every agent such that they can now deduce the new swarming point from the received robot states. This calibration now required to install $Q$ is necessary to specify the full properties of the system $J$ (sufficient proximity for communication, and rule specification), implying that $J$ dominates $Q$ in calibration complexity according to Definition 5.1. NOTE: Both algorithms require the same additional initial calibration that they be sufficiently far away from an obstacle in order to maintain such a state, but this does not alter the analysis.

3. \textit{mb($Q$) is bounded above by mb($J$)}:
    Taking all the components (vertices) of the system (graph) as \textit{black box sensor systems} or \textit{robots} where the only information potential know regarding the system is the output size $b$, then the maximum bandwidth of such systems is bounded above by $b$ according to Definition 6.1. Since all robots in both systems are assumed to be homogeneous and all explicitly communicated information shares the same \textit{output size} $log \mathbb{K}(b)$, they have the same maximum bandwidth and so it follows that $mb(J) = mb(Q)$.
\end{proof}




This result, together with our experimental findings demonstrating that individual dynamic coverage covered less distance and converged faster in every map scenario, shows that our swarm system $S$ with less information complexity than our multi-robot dynamic coverage system $M$ is able to constantly outperform the latter in these two metrics. A conclusion can be drawn that adding information complexity in the form of communication to a multiple robot dynamic coverage system does not always increase its performance with regards to convergence time and distance travelled.

\subsection{Navigation}

Experimental results demonstrated that DMA-RRT converges faster and travels less distance in cluttered environments, meanwhile PBC performs best overall on empty maps, and PF yields the smallest number of collisions on scenarios in which it converges. We will attempt to establish a reduction between these three algorithms in order to offer a better understanding of the experimental findings.


The circuit models for the three navigation algorithms \textit{PF}, \textit{Barrier Certificates}, and \textit{DMA-RRT} will be denoted as $P$, $B$, and $DMA$ for short. Since all three navigation algorithms are guaranteed to achieve the same limit sets (final configurations), they \textit{simulate} each other as computational circuits:
\begin{align}
B \cong P \cong DMA
\end{align}
where the $\cong$ (simulation) operation is elementary transitive \cite{donald1995information}.

Work remains in showing that the DMA-RRT algorithm reduces to any of the barrier certificates or PF algorithms, and this can be done by meeting the requirements of Definition 6.2 \cite{donald1995information}.

This leads us to an attempt at forming a reduction as outlined in Definition 3.12 \cite{donald1997information}:

1. DMA-RRT and barrier certificates \textit{simulate} each other by virtue of having the same limit sets.

2. Barrier Certificates dominates DMA-RRT in calibration complexity because it has the base requirement that robots must be located within a sufficient distance from an obstacle in order for the safe set to be forward invariant. DMA-RRT, however, has no such constraints.

3. mb(DMA-RRT) is \textbf{not} bounded above by mb(BC) because DMA-RRT communicates an entire path ($2DOF \times num\_waypoints$), whereas Barrier Certificates obstacle sensing produces only a new velocity $v$ for a robot to reactively avoid a collision ($2DOF$). Since all robots in both algorithms are seen as black-box sensor sytems with external communication, their maximum bandwidth is denoted by the maximum of their output size and their external bandwidth as in Definition 6.1 \cite{donald1995information}.

The converse, Barrier Certificates reducing to DMA-RRT, also does not hold because the calibration complexity of DMA-RRT does not dominate that of Barrier Certificates. A similar argument can be made for the PF algorithm, hence both swarm navigation algorithms are not provably reducible to the multirobot navigation algorithms and vice versa.

This finding suggests that the experimental results relating the performances of the navigation algorithms, differently from those of the dynamic coverage algorithms, do not relate simply through an information theoretic reduction. It appears that the information complexity of DMA-RRT is fundamentally different from that of PBC and PF, offering possible explanations for the advantages of each algorithm in different scenarios.




\section{Conclusion}
\label{sec:conclusion}
Our research provided a comparison for identifying and highlighting the pros and cons of swarm and  general multi-robot algorithms in a manner that can inform offline design decisions made by engineers  or  online operational decisions made by supervisory operators of multi-robot systems. We first conducted an evaluation on the ROS platform using five algorithms in order to compare multi-robot and swarm systems in two application domains: navigation and dynamic area coverage. We then present several insights based on our collision, coverage, distance, and convergence metrics. Lastly, we extend our results to a comparative analysis of the algorithms based on the theory of information invariants, which provided a theoretical characterization supported by our empirical results.

\bibliographystyle{IEEEtran}
\bibliography{references}

\end{document}